\documentclass[]{article}
\usepackage{times}
\usepackage{graphicx}
\usepackage{amsmath}
\usepackage{amsfonts}
\usepackage{amsthm}
\usepackage{booktabs}
\usepackage{enumitem}
\usepackage[numbers]{natbib}
\usepackage{xcolor}
\usepackage[hidelinks]{hyperref}
\usepackage{doi}

\usepackage{microtype}
\usepackage{lmodern}
\usepackage{setspace}
\usepackage{titlesec}
\titleformat{\section}{\large\bfseries}{\thesection}{1em}{}
\titleformat{\subsection}{\normalsize\bfseries}{\thesubsection}{1em}{}

\setlist{itemsep=0.25em, topsep=0.5em}

\usepackage[all]{nowidow}

\DeclareMathOperator*{\argmin}{arg\,min}

\newtheorem{theorem}{Theorem}
\newtheorem{proposition}{Proposition}
\newtheorem{remark}{Remark}

\newtheorem{assumption}{Assumption}
\newtheorem{fact}{Fact}

\title{Consequences of Kernel Regularity for \\Bandit Optimization}
\author{Madison Lee \\ \small{\texttt{mal017@ucsd.edu}} 
\and 
Tara Javidi \\ \small{\texttt{tjavidi@ucsd.edu}}}
\date{}

\begin{document}

\maketitle

\begin{abstract}%
In this work we investigate the relationship between kernel regularity and algorithmic performance in the bandit optimization of RKHS functions. While reproducing kernel Hilbert space (RKHS) methods traditionally rely on global kernel regressors, it is also common to use a smoothness-based approach that exploits local approximations. We show that these perspectives are deeply connected through the spectral properties of isotropic kernels. In particular, we characterize the Fourier spectra of the Mat\'ern, square-exponential, rational-quadratic, $\gamma$-exponential, piecewise-polynomial, and Dirichlet kernels, and show that the decay rate determines asymptotic regret from both viewpoints. For kernelized bandit algorithms, spectral decay yields upper bounds on the maximum information gain, governing worst-case regret, while for smoothness-based methods, the same decay rates establish H\"older space embeddings and Besov space norm-equivalences, enabling local continuity analysis. These connections show that kernel-based and locally adaptive algorithms can be analyzed within a unified framework. This allows us to derive explicit regret bounds for each kernel family, obtaining novel results in several cases and providing improved analysis for others. Furthermore, we analyze LP-GP-UCB, an algorithm that combines both approaches, augmenting global Gaussian process surrogates with local polynomial estimators. While the hybrid approach does not uniformly dominate specialized methods, it achieves order-optimality across multiple kernel families.
\end{abstract}

\section{Introduction}
We are interested in maximizing a black-box function using active sample selection. A function $f:\mathcal{X}\subseteq\mathbb{R}^d\rightarrow\mathbb{R}$ is considered a black-box function when we can only access it using a zero-order oracle for $f$ that returns a noisy evaluation $y_x=f(x)+\eta_x$ when queried at a point $x\in\mathcal{X}$. We want to design an algorithm that sequentially selects the query points $\{x_i\}_{i=1}^n$ in a way that minimizes the expectation of the cumulative regret $\mathcal{R}_n$, defined as 
\[\mathcal{R}_n=\sum_{i=1}^n f(x^*)-f(x_i),\]
where $x^*$ is the maximizer of $f$.

In order to make this problem meaningful, we must assume that $f$ belongs to a class of functions such that previously selected samples are informative. This requirement commonly imposes the containment of $f$ in a well-structured space such as a H\"older space or reproducing kernel Hilbert space (RKHS). 
Optimization in an RKHS, following the seminal framework proposed in \cite{srinivasGaussianProcessOptimization2010}, uses the samples globally, leveraging a Gaussian process (GP) surrogate model for kernel regression that is finite-dimensional due to the representer theorem \citep{scholkopfGeneralizedRepresenterTheorem2001}. This is in contrast to the optimization of continuous functions with smoothness properties such as H\"older \citep{liuSmoothBanditOptimization2021} or Besov regularity \citep{singhContinuumArmedBanditsFunction2021} which relies on a locally adaptive approximation process that typically remains nonlinear.

In this work we focus on the RKHS setting and demonstrate that many common isotropic kernel functions have Fourier spectra with decaying tails. 
This notion of regularity allows us to view bandit optimization in these RKHSs from two different perspectives, a global interpolation one, common in the kernelized bandit literature, and less intuitively, a local approximation one that arises in the optimization of smooth H\"older and Besov spaces, where only proximal samples are needed to generate an optimization surrogate. For kernelized bandit algorithms, spectral decay yields upper bounds on the maximum information gain, governing worst-case regret, while for smoothness-based methods, the same decay rates establish H\"older space embeddings and Besov space norm-equivalences, enabling local continuity analysis. Algorithmically, the connections we draw yield a unified framework in which both globally coupled kernel regression algorithms, such as SupKernelUCB \citep{valkoFiniteTimeAnalysisKernelised2013a}, and locally adaptive procedures, such as Meta-UCB \citep{liuSmoothBanditOptimization2021}, can be applied and analyzed comparatively. Furthermore, it sheds new light on LP-GP-UCB \citep{leeMultiScaleZeroOrderOptimization2022a}, an algorithm that augments the global kernel regressors with locally adaptive polynomial estimators. 

The rest of the paper is organized as follows. We first introduce the notations used in this paper in Section~\ref{sec:preliminaries}. We then provide an overview of our contributions in Section~\ref{sec:contributions} and conclude the section with a discussion of background information and related works in Section~\ref{sec:relatedwork}.
Section~\ref{sec:mainanalysis} provides the main results of the paper. In particular, Section~\ref{subsec:analysisspectral} establishes the spectral decay of specific kernels. Using these spectral decay rates, Section~\ref{subsec:global} develops the results from the global, kernelized perspective, and Section~\ref{subsec:local} develops the results for the local smoothness perspective. Section~\ref{subsec:algorithm} explores a hybrid global-local perspective, and Section~\ref{sec:discussion} contains a discussion of the results and final remarks.

\subsection{Preliminaries}
\label{sec:preliminaries}
We first present an overview of the notations used in this paper. The precise definitions of these objects and properties are given in Appendix~\ref{appendix:definitions}. 
\begin{itemize}
    \item The objective function $f$ maps $\mathcal{X}=[0,1]^d$ to $\mathcal{Y}=\mathbb{R}$. $f$ can be accessed through noisy evaluations $y_x=f(x)+\eta_x$, where $x\in\mathcal{X}$ and the additive noise $\eta_x$ is assumed to be $\sigma^2$-subgaussian. 

    \item Given a positive-definite kernel $k$, we shall use the term $\mathcal{H}_k$ and $\lVert\cdot\rVert_k$ to denote the RKHS associated with $k$ and the corresponding RKHS norm. When $k$ is isotropic, depending only on the distance between its arguments, we overload the notation and write $k(\lVert x-y\rVert)=k(x,y)$.
    
    \item We let $k_{\nu}$ be the Mat\'ern kernel with parameter $\nu>0$, $k_{\text{SE}}$ the square-exponential kernel, $k_{\text{RQ}}$ the rational-quadratic kernel, $k_{\gamma-\text{Exp}}$ the $\gamma$-exponential kernel with parameter $\gamma\in(0,2]$, $k_{\text{PP}-q}$ the piecewise-polynomial kernel with parameter $q\in\mathbb{Z}_{\ge0}$, and $k_{\text{PBL}}$ the Dirichlet kernel.
    
    \item For $\alpha>0$, we use $\mathcal{C}^{\alpha}$ and $\lVert\cdot\rVert_{\mathcal{C}^{\alpha}}$to denote the H\"older (H\"older-Zygmund) space of order $\alpha$ and the corresponding norm.

    \item For $s>0$ and $1\le p,q\le\infty$, we use $\mathcal{B}^s_{p,q}$ to denote the Besov space with smoothness $s$, integrability parameter $p$, and smoothness scaling parameter $q$ and $\lVert\cdot\rVert_{\mathcal{B}^s_{p,q}}$ to denote the corresponding norm.
    
    \item We use $\tilde{\mathcal{O}}\left(\cdot\right)$ to represent asymptotic upper bounds that hide polylogarithmic factors and $\tilde{\Omega}\left(\cdot\right)$ for asymptotic lower bounds that hide polylogarithmic factors.
    
\end{itemize}
\subsection{Bandit Optimization}
\label{sec:relatedwork}

The optimization of RKHS functions from noisy samples was formulated as the kernelized bandits problem in \cite{srinivasGaussianProcessOptimization2010}, where the GP-UCB algorithm was proposed based on the upper confidence bound (UCB) strategy for multi-armed bandits \citep{auerFinitetimeAnalysisMultiarmed2002}. In \cite{srinivasInformationTheoreticRegretBounds2012}, GP-UCB was shown to achieve an asymptotic worst-case regret of $\tilde{\mathcal{O}}\left(\gamma_n\sqrt{n}\right)$ in terms of the maximum information gain $\gamma_n$, the maximum mutual information between all sets of $n$ noisy observations and the underlying function. $\gamma_n$ is known to scale with the effective dimensionality of the kernel \citep{calandrielloGaussianProcessOptimization2019} and arises almost universally in the regret analysis of kernelized bandit algorithms that leverage globally optimal kernel ridge regressors as optimization surrogates.

Following GP-UCB, a variety of algorithms have been developed which have achieved improved worst-case regret upper bounds in different ways. In the case of finite $|\mathcal{X}|$, algorithms such as SupKernelUCB \citep{valkoFiniteTimeAnalysisKernelised2013a}, RIPS \citep{camilleriHighDimensionalExperimentalDesign2021a}, GP-ThreDS \citep{salgiaDomainShrinkingBasedBayesian2021}, and BPE \citep{liGaussianProcessBandit2022b} have achieved an improved regret upper bound of $\tilde{\mathcal{O}}\left(\sqrt{\gamma_n n}\right)$ by selectively restricting the optimization domain in different ways. Furthermore, following the argument in \cite{liGaussianProcessBandit2022b}, this improved bound can also be achieved in the continuous case if the kernel satisfies an additional Lipschitz smoothness constraint, which had been originally shown to be satisfied for many isotropic kernels in \cite{shekharMultiScaleZeroOrderOptimization2020a}. An explicit lower bound on worst-case regret of $\Omega(n^{\frac{\nu+d}{2\nu+d}})$ for the Mat\'ern~kernel and $\Omega(n^{\frac{1}{2}}\log^{\frac{d}{2}}n)$ for the square-exponential kernel was derived in \cite{scarlettLowerBoundsRegret2017}. Since upper bounds on the maximum information gain corresponding to these kernels were determined in \cite{vakiliInformationGainRegret2021}, any algorithm achieving $\tilde{\mathcal{O}}\left(\sqrt{\gamma_n n}\right)$ regret is thus known to be order-optimal, up to polylogarithmic factors, for the Mat\'ern~kernel, for which $\gamma_n=\mathcal{O}(n^{\frac{d}{2\nu+d}}\log^{\frac{2\nu}{2\nu+d}}(n))$ for $\nu>1$ and $\gamma_n=\mathcal{O}(\log^{d+1}(n)
)$ for the SE kernel. These upper bounds on $\gamma_n$ rely on tail bounds for the kernel spectra, which are contingent on the smoothness of the kernel function itself.

Motivated by the apparent need for sufficient levels of smoothness in the regret analysis for continuous RKHS functions, we turn to the literature on optimization in higher-order smooth function spaces to evaluate whether the global kernel structure provides any benefit over algorithms that rely on local notions of continuity alone. For the Mat\'ern RKHS, lower bounds on the worst-case regret matching the ones in \cite{scarlettLowerBoundsRegret2017} were recovered in \cite{singhContinuumArmedBanditsFunction2021} by the equivalence of the Mat\'ern RKHS with certain Besov spaces. In \cite{singhContinuumArmedBanditsFunction2021}, the authors find that the order-optimal algorithms designed for optimizing functions in H\"older~spaces, such as such as UCB-Meta from \cite{liuSmoothBanditOptimization2021}, are sufficient for matching the lower bound for Besov spaces, and thus are order-optimal for the Mat\'ern RKHS. This connection between the kernelized bandit optimization problem and the H\"older~continuous function optimization problem for the Mat\'ern~RKHS with $\nu+\frac{1}{2}\in\mathbb{N}$ was made explicit in \cite{liuAdaptationMisspecifiedKernel2023b}. This connection relies on a Sobolev embedding theorem that requires a half-integer smoothness parameter, leading the authors to require $\nu+\frac{1}{2}\in\mathbb{N}$. However, by applying a fractional Sobolev embedding theorem, e.g. Thm. 3.6.2 in \cite{saloFunctionSpaces2008}, the authors' statement can be seen to hold more generally for $\nu>0$. For the Mat\'ern kernel, unlike with the information-based results for kernelized bandit algorithms, which hold only for $\nu>1$, the smoothness algorithms obtain regret upper bounds for the full parameter range $\nu>0$. 

\subsection{Overview of Results}
\label{sec:contributions}
\begin{table*}[ht]
\caption{Regret Bounds for Global, Local, and Hybrid Algorithms}
\begin{center}
\begin{tabular}{lcccc}
\toprule
 \textbf{Space} & \textbf{Lower Bound} & \multicolumn{3}{c}{\textbf{Upper Bound}}\\\cmidrule{3-5}
 & &  \textbf{SupKernelUCB} & \textbf{UCB-Meta} & \textbf{LP-GP-UCB}\\  
\midrule
$k_{\text{SE}}$ & { $\tilde{\Omega}\left(\sqrt{n}\right)$} & {\color{blue} $\tilde{\mathcal{O}}\left(\sqrt{n}\right)$} & {\color{blue} $\tilde{\mathcal{O}}\left(\sqrt{n}\right)$} & {\color{blue} $\tilde{\mathcal{O}}\left(\sqrt{n}\right)$}\\
$k_{\nu\le1}$ & { $\tilde{\Omega}\left( n^{\frac{\nu + d}{2\nu + d}}\right)$} & $\tilde{\mathcal{O}}\left(n^{\frac{2\nu+d}{4\nu}}\right)$ & {\color{blue} $\tilde{\mathcal{O}}\left(n^{\frac{\nu+d}{2\nu+d}}\right)$} & {\color{blue} $\tilde{\mathcal{O}}\left(n^{\frac{\nu+d}{2\nu+d}}\right)$}\\
$k_{\nu>1}$ & { $\tilde{\Omega}\left( n^{\frac{\nu + d}{2\nu + d}}\right)$} & {\color{blue} $\tilde{\mathcal{O}}\left(n^{\frac{\nu + d}{2\nu + d}}\right)$} & {\color{blue} $\tilde{\mathcal{O}}\left(n^{\frac{\nu+d}{2\nu+d}}\right)$} & $\tilde{\mathcal{O}}\left(n^{\frac{2\nu+3d}{4\nu+2d}}\right)$\\
$k_{\text{RQ}}$ & DNE$^{(1)}$  & {\color{blue} $\tilde{\mathcal{O}}\left(\sqrt{n}\right)$} & {\color{blue} $\tilde{\mathcal{O}}\left(\sqrt{n}\right)$} & {\color{blue} $\tilde{\mathcal{O}}\left(\sqrt{n}\right)$}\\
$k_{\gamma-\text{Exp}}$ & { $\tilde{\Omega}\left( n^{\frac{\gamma + 2d}{2\gamma + 2d}}\right)$} & $\tilde{\mathcal{O}}\left(n^{\frac{\gamma+d}{2\gamma}}\right)$ & {\color{blue} $\tilde{\mathcal{O}}\left( n^{\frac{\gamma + 2d}{2\gamma + 2d}}\right)$} & {\color{blue} $\tilde{\mathcal{O}}\left( n^{\frac{\gamma + 2d}{2\gamma + 2d}}\right)$}\\
$k_{\text{PP-}q}$ & { $\tilde{\Omega}\left( n^{\frac{2q+1 + 2d}{4q+2 + 2d}}\right)$} & DNE & {\color{blue}$\tilde{\mathcal{O}}\left(n^{\frac{2q+1+2d}{4q+2+2d}}\right)$} & $\tilde{\mathcal{O}}\left(n^{\frac{2q+d}{2q+1+d}}\right)$\\
$k_{\text{PBL}}$ & DNE & {\color{blue} $\tilde{\mathcal{O}}\left(\sqrt{n}\right)$} &  {\color{blue} $\tilde{\mathcal{O}}\left(\sqrt{n}\right)$} & {\color{blue} $\tilde{\mathcal{O}}\left(\sqrt{n}\right)$}\\
\bottomrule
\end{tabular}
\footnotesize
\begin{enumerate}[label=(\arabic*),noitemsep]
    \item DNE indicates the regret bounds have not been shown explicitly for the entire RKHS. The bounds may be tightened under specific conditions on the parameters for some kernels (see Section~\ref{subsubsec:infogainregret} and \ref{subsec:regretanalysis}).
\end{enumerate}
\label{table:specificregbounds}
\end{center}
\end{table*}

First we present the assumptions on the objective function that inform our analysis.
\begin{assumption}
\label{assump:main}
We make the following assumptions on the objective function $f$ and the observation oracle:
    \begin{itemize}
        \item $f\in\mathcal{H}_k$ for some known kernel $k$ and $\lVert f\rVert_k\le B$ for some constant $B>0$.
        \item $k$ is isotropic, that is, $k(x,y)$ is a function of $\lVert x-y\rVert$.
        \item The observation noise $\{\eta_i\}_{i\ge0}$ is i.i.d. and $\sigma^2$-subgaussian with $\sigma^2>0$.
    \end{itemize}
\end{assumption}
These assumptions ensure that function evaluation is continuous, the function is Fourier-transformable, and the noise distribution has tails that decrease at least as fast as a Gaussian.
Next we list the main contributions of this paper:
\begin{itemize}
    \item We characterize the spectra of the Mat\'ern, square-exponential, rational-quadratic, $\gamma$-exponential, piecewise-polynomial, and Dirichlet kernels and show that they decay at least polynomially fast in the limit (Proposition~\ref{prop:spectraldecay}, Section~\ref{subsec:analysisspectral}).
    \item We use these spectral characterizations to derive upper bounds on $\gamma_n$, the maximum information gain of each kernel (Proposition~\ref{prop:infogain}, Section~\ref{subsec:analysisinfogain}). We show that our spectral analysis provides a worst-case global interpolation error bound that ultimately determines upper bounds on $\gamma_n$, exposing the underlying mechanics behind the kernel eigenanalysis in \cite{vakiliInformationGainRegret2021} and facilitating the derivation of $\gamma_n$ bounds for a number of isotropic kernels beyond the Mat\'ern and square-exponential shown in \cite{vakiliInformationGainRegret2021}. As a result, we obtain explicit characterizations, novel in several cases, of the regret incurred by kernelized bandit algorithms such as SupKernelUCB \citep{valkoFiniteTimeAnalysisKernelised2013a} whose performance depends on $\gamma_n$ (Section~\ref{subsubsec:infogainregret}).
    \item We show that RKHSs with at least polynomial decay are embedded in H\"older spaces, with order $\alpha$ dictated by the decay rate (Proposition~\ref{prop:holder}, Section~\ref{subsec:analysisholder}). This embedding result reveals that one may optimize the functions in these RKHSs using optimization algorithms designed for H\"older spaces, which use the H\"older smoothness parameters to compute local polynomial approximations and obtain upper bounds on the worst-case regret \citep{liuSmoothBanditOptimization2021}. This allows us to explicitly upper bound the regret for bandit algorithms designed for H\"older spaces when applied to RKHS functions (Section~\ref{subsubsec:holderregret}).
    \item We show that RKHSs satisfying Assumption~\ref{assump:main} whose spectra decay like a polynomial in the limit are norm-equivalent to Besov spaces (Proposition~\ref{prop:besov}, Section~\ref{subsec:analysisbesov}). This result establishes an equivalence, known for the Mat\'ern RKHS \citep{sawanoTheoryBesovSpaces2018,singhContinuumArmedBanditsFunction2021} and piecewise polynomial RKHS \citep{wendlandErrorEstimatesInterpolation1998}, for the first time between the $\gamma$-exponential RKHS and Besov spaces. This allows us to obtain lower bounds on the worst-case regret for the $\gamma$-exponential and piecewise polynomial kernels for the first time, and provides an alternative proof for the Mat\'ern kernel \citep{scarlettLowerBoundsRegret2017}. This also allows us to utilize known bandit algorithms such as UCB-Meta \citep{liuSmoothBanditOptimization2021} and achieve optimal regret (Section~\ref{subsubsec:besovregret}).

    \item Motivated by the generality and case-dependent optimality of H\"older-optimal algorithms, we analyze LP-GP-UCB, an algorithm proposed in \cite{leeMultiScaleZeroOrderOptimization2022a} that augments Gaussian process (GP) surrogate models with LP approximations to exploit the existing smoothness properties of RKHS functions. We use our results to specialize the generic regret bounds for the LP-GP-UCB algorithm \citep{leeMultiScaleZeroOrderOptimization2022a}, which depend on the \textit{maximum information gain} $\gamma_n$ of the kernel $k$ and the H\"older smoothness parameter $\alpha$. We improve upon the former analysis of LP-GP-UCB and obtain upper bounds that are explicit in $n$ for the Mat\'ern, square-exponential, rational-quadratic, $\gamma$-exponential, piecewise-polynomial, and Dirichlet kernels (Theorem~\ref{thm:lpgpucbregret}, Section~\ref{subsec:regretanalysis}).
\end{itemize}
The regret bounds for specific kernels are summarized in Table~\ref{table:specificregbounds}, with the SupKernelUCB column representing GP-based kernelized bandit algorithms achieving $\mathcal{R}_n=\tilde{\mathcal{O}}\left(\sqrt{n\gamma_n}\right)$, UCB-Meta representing H\"older-smoothness-based algorithms achieving $\mathcal{R}_n=\tilde{\mathcal{O}}\left(n^{\frac{\alpha+d}{2\alpha+d}}\right)$, and LP-GP-UCB representing the hybrid approach leveraging both kernel structure and smoothness. The tightest bounds across all algorithms are highlighted in blue. The lower bounds for the Mat\'ern and square-exponential RKHSs have been shown before \citep{scarlettLowerBoundsRegret2017}, while those for the $\gamma$-exponential and piecewise-polynomial RKHSs are available due to our Besov equivalence result.

\section{Main Results}
\label{sec:mainanalysis}

\subsection{Kernel Spectrum Analysis}
\label{subsec:analysisspectral}
We begin the analysis by showing that the Mat\'ern, square-exponential, rational-quadratic, $\gamma$-exponential, piecewise-polynomial, and Dirichlet kernels have Fourier transforms whose tails decay rapidly. The definitions of these kernels may be found in Appendix~\ref{appendix:definitions}.

\begin{proposition}[Spectral decay of specific isotropic kernels]
\label{prop:spectraldecay}
For the square-exponential, rational-quadratic, and Dirichlet kernels, there exist finite constants $C_1,C_2>0$ such that their Fourier transforms $\hat{k}(\omega)$ decay exponentially:
\[\hat{k}(\omega)\le C_1\exp(-C_2\lVert\omega\rVert_2),\;\lVert\omega\rVert_2\rightarrow\infty.\]
For the Mat\'ern, $\gamma$-exponential, and piecewise-polynomial (with $q\ge1$ if $d=1,2$) kernels, there exist finite constants $C_1,C_2>0$ and $\tau>\frac{d}{2}$ such that their Fourier transforms $\hat{k}(\omega)$ decay polynomially:
\[C_1(1+\lVert\omega\rVert_2)^{-\tau}\le\hat{k}(\omega)\le C_2(1+\lVert\omega\rVert_2)^{-\tau},\;\lVert\omega\rVert_2\rightarrow\infty.\]
In particular, for the Mat\'ern kernel, $\tau=2\nu+d$, for the $\gamma$-exponential kernel, $\tau=\gamma+d$, and for the piecewise-polynomial kernel with order $q$, $\tau=2q+1+d$
The upper bound holds for the piecewise polynomial kernel in general.

For the Dirichlet kernel, the Fourier transform is compactly supported.
\end{proposition}
The proof of this statement, given in Appendix~\ref{appendix:decayproof}, is given for each kernel separately, and relies on either asymptotic tail bounds on the Fourier transform or direct computation when possible. 

\begin{remark} 
These decay results are new for the $\gamma$-exponential kernel, straightforward for the rational-quadratic and Dirichlet kernels, and well-known for the square-exponential and Mat\'ern kernels, as they are necessary to obtain the results in \cite{vakiliInformationGainRegret2021}. 
\end{remark}

In the following sections, we show that these spectral decay results determine the fundamental limits on the performance of optimization algorithms from two different perspectives, a global interpolation one found in the kernelized bandit literature, and a local approximation one that arises in the optimization of smooth spaces such as the Besov and H\"older spaces.

\subsection{Global Structure via Information Gain Bounds}
\label{subsec:global}
In this section, we consider the consequence of the spectral decay rate on fundamental limits in global interpolation and then obtain regret bounds for kernelized bandit algorithms. 

\subsubsection{Maximum Information Gain Analysis}
\label{subsec:analysisinfogain}

We first use the spectral decay results to derive novel information gain upper bounds for some RKHSs and improve existing analysis in certain kernel regimes. In \cite{vakiliInformationGainRegret2021}, specific information gain bounds were derived for the Mat\'ern and square-exponential kernels. We improve the analysis for the Mat\'ern kernel and derive new information gain bounds for the rational-quadratic, $\gamma$-exponential, piecewise-polynomial, and Dirichlet kernels.

\begin{proposition}[Information Gain Bounds for Kernels with Decaying Spectrum]
\label{prop:infogain}
    For an RKHS associated to an isotropic, positive-definite kernel $k$ whose Fourier transform decays polynomially with rate $\tau=\beta+d$, where $\beta>\frac{d}{2}$, the maximum information gain satisfies
    \[\gamma_n=\mathcal{O}\left(n^{\frac{d}{\beta}}\log^{\frac{\beta-d}{\beta}}(n)\right).\]
    If it is also true that $\beta\ge 1$ and $d$ is odd, or $\beta\ge 2$, then
    \[\gamma_n=\mathcal{O}\left(n^{\frac{d}{\beta+d}}\log^{\frac{\beta}{\beta+d}}(n)\right).\]
    For an RKHS associated to an isotropic, positive-definite kernel $k$ whose Fourier transform decays exponentially,
    \[\gamma_n=\mathcal{O}\left(\log^{d+1}(n)\right).\]
    For an RKHS associated to an isotropic, positive-definite kernel $k$ whose Fourier transform is compactly supported,
    \[\gamma_n=\mathcal{O}\left(\log(n)\right).\]
\end{proposition}
The full proof of this statement is given in Appendix~\ref{appendix:infogainproof} and the specific bounds for different kernels based on the decay results of Proposition~\ref{prop:spectraldecay} are summarized in Table~\ref{table:infogainbounds}. 

\textit{Proof Outline:} Mercer's theorem (e.g., Theorem 4.2, \cite{rasmussenGaussianProcessesMachine2006}), states that a positive definite kernel $k$ may be expressed in terms of absolutely summable Mercer eigenvalues $\lambda_i>0$ and eigenfunctions $\phi_i$:
\[k(x,y)=\sum_{i=1}^{\infty}\lambda_i\phi_i(x)\phi_i^*(y).\]

These eigenvalues characterize the fundamental limits of $L_2$ function approximation in finite-dimensional subspaces of RKHSs, and can be bounded using the decay of the kernel's Fourier transform \citep{schabackApproximationPositiveDefinite2003}. Thus, our spectral decay results allow us to deduce upper bounds on the kernels' Mercer eigenvalues directly, using a result from \cite{schabackApproximationPositiveDefinite2003} which we strengthen using error bounds from \cite{narcowichSobolevErrorEstimates2006}. Using these eigenvalue tail bounds and the results of Proposition~\ref{prop:spectraldecay}, we can then derive specific information gain upper bounds using the approach of \cite{vakiliInformationGainRegret2021}, where it was shown that one may derive upper bounds on $\gamma_n$ for kernels whose Mercer eigenvalues decay sufficiently rapidly.
\begin{table}[ht]
\centering
        \caption{Information gain bounds for different RKHSs.}

\label{table:infogainbounds}
    \begin{tabular}{@{}lc@{}}

    \toprule
    Kernel    & $\gamma_n$ \\ \midrule
    Mat\'ern               &   $\mathcal{O}\left(n^{\frac{d}{2\nu}}\log^{\frac{2\nu-d}{2\nu}}(n)\right)$        \\
    Mat\'ern, $\nu\ge1$, or $\nu\ge\frac{1}{2}$ and $d$ odd &  $\mathcal{O}\left(n^{\frac{d}{2\nu+d}}\log^{\frac{2\nu}{2\nu+d}}(n)\right)$        \\
    Square-Exponential                            &     $\mathcal{O}\left(\log^{d+1}(n)\right)$        \\
    Rational-Quadratic                            &     $\mathcal{O}\left(\log^{d+1}(n)\right)$         \\
    $\gamma$-Exponential                         &  $\mathcal{O}\left(n^{\frac{d}{\gamma}}\log^{\frac{\gamma-d}{\gamma}}(n)\right)$   \\
    $\gamma$-Exponential, $\gamma\in[1,2]$ and $d$ odd                         &  $\mathcal{O}\left(n^{\frac{d}{\gamma+d}}\log^{\frac{\gamma}{\gamma+d}}(n)\right)$   \\
    Piecewise-Polynomial, $q\ge1$ for $d\in\{1,2\}$ &  $\mathcal{O}\left(n^{\frac{d}{2q+1+d}}\log^{\frac{2q+1}{2q+1+d}}(n)\right)$  \\
    Dirichlet                     &        $\mathcal{O}\left(\log(n)\right)$        \\
    \bottomrule
    \end{tabular}
    \end{table}

\subsubsection{Information-Based Regret Bounds}
    \label{subsubsec:infogainregret}
The maximum information gain $\gamma_n$ is important for the regret analysis of kernelized, Gaussian process (GP) bandit algorithms that rely on GP surrogate models in the sample selection process. $\gamma_n$ arises in the regret analysis due to the close relationship between the mutual information and the sum of the GP posterior variances, which are used to bound the concentration of the GP posterior mean around the true function at every step of the sequential optimization process \citep{srinivasInformationTheoreticRegretBounds2012}. The best-known kernelized bandit algorithms, following SupKernelUCB \citep{valkoFiniteTimeAnalysisKernelised2013a}, that operate in an RKHS ball with samples subject to sub-Gaussian noise have cumulative regret upper bounded by $\tilde{\mathcal{O}}(\sqrt{n\gamma_n})$ with high probability. Using our specific information gain bounds, we can upper bound the asymptotic cumulative regret of these algorithms explicitly in terms of the number of samples. The resulting regret upper bounds are given in Table~\ref{table:inforegretbounds}, modulo poly-logarithmic factors. Note that there are some additional cases represented beyond the maximally general results displayed in Table~\ref{table:specificregbounds}, as information-gain bounds may exist in specific parameter regimes but not in general. In general kernelized bandit algorithm performance improves with increased smoothness, but can suffer greatly in high-dimensional regimes with low levels of smoothness, as seen in the general case for the Mat\'ern and $\gamma$-exponential kernels.
\begin{table}[ht]
\centering
        \caption{Information-based regret bounds for different RKHSs.}
\label{table:inforegretbounds}
    \begin{tabular}{@{}lc@{}}

    \toprule
    Kernel    & $\mathcal{R}_n$ \\ \midrule
    Mat\'ern               &   $\tilde{\mathcal{O}}\left(n^{\frac{2\nu+d}{4\nu}})\right)$        \\
    Mat\'ern, $\nu\ge1$, or $\nu\ge\frac{1}{2}$ and $d$ odd &  $\tilde{\mathcal{O}}\left(n^{\frac{\nu+d}{2\nu+d}}\right)$        \\
    Square-Exponential                            &     $\tilde{\mathcal{O}}\left(\sqrt{n}\right)$        \\
    Rational-Quadratic                            &     $\tilde{\mathcal{O}}\left(\sqrt{n}\right)$         \\
    $\gamma$-Exponential                         &  $\tilde{\mathcal{O}}\left(n^{\frac{\gamma+d}{2\gamma}}\right)$   \\
    $\gamma$-Exponential, $\gamma\in[1,2]$ and $d$ odd                         &  $\tilde{\mathcal{O}}\left(n^{\frac{\gamma+2d}{2\gamma+2d}}\right)$   \\
    Piecewise-Polynomial, $q\ge1$ for $d\in\{1,2\}$ &  $\tilde{\mathcal{O}}\left(n^{\frac{2q+1+2d}{4q+2+2d}}\right)$  \\
    Dirichlet                     &        $\tilde{\mathcal{O}}\left(\sqrt{n}\right)$        \\
    \bottomrule
    \end{tabular}
    \end{table}

\subsection{Local Structure via H\"older Embeddings and Besov Equivalence}
\label{subsec:local}

In the previous section we found that the decay of a kernel's Fourier transform determines the expressivity of the kernel in sample interpolation and consequently the worst-case performance of RKHS function optimization algorithms that leverage global GP regressors. However, Fourier decay is also deeply tied to the local continuity properties of a function and its higher-order derivatives. 

\subsubsection{H\"older Space Embeddings}
\label{subsec:analysisholder}
First, we present an embedding result which says that we can identify elements of the RKHSs associated with the Mat\'ern, square-exponential (SE), rational-quadratic (RQ), $\gamma$-exponential ($\gamma$-Exp), piecewise-polynomial, and Dirichlet kernels with elements of certain H\"older spaces. In particular, we use the spectral decay results of Proposition~\ref{prop:spectraldecay} to characterize the higher-order smoothness of our isotropic RKHSs, allowing us to embed them into H\"older spaces.

\begin{proposition}[H\"older Embeddings for Kernels with Sufficient Spectral Decay]
    \label{prop:holder}
    If $f$ is in the RKHS $\mathcal{H}_k$ associated to a kernel $k$ that has a Fourier transform $\hat{k}(\omega)$ satisfying $\hat{k}(\omega)\le C_1(1+\lVert\omega\rVert_2)^{-(\beta+d)}$ for all $\omega$ and some $C_1>0$ and $\beta>0$, then there exists a constant $C_2$ such that $\lVert f\rVert_{\mathcal{C}^{\alpha}}\le C_2\lVert f\rVert_{\mathcal{H}_k}$, where $\mathcal{C}^{\alpha}$ is a H\"older space with order $\alpha=\frac{\beta}{2}$.
\end{proposition}
The proof of this statement is given in Appendix~\ref{appendix:holder}. This result, in combination with the spectral decay results of Proposition~\ref{prop:spectraldecay}, shows that the Mat\'ern, SE, RQ, $\gamma$-Exp, PP, and Dirichlet RKHSs are embedded in H\"older spaces. Note that due to the exponential decay of the SE, RQ, and Dirichlet RKHSs, these RKHSs are embedded in any H\"older space with a finite smoothness parameter $\beta$.

\begin{table}[ht]
\centering
        \caption{H\"older~smoothness parameters for different RKHSs.}
\label{table:holderparams}
    \begin{tabular}{@{}ll@{}}

    \toprule
    Kernel    & $\alpha$  \\ \midrule
    Mat\'ern-$\nu$               & $\nu$   \\
    SE                            &  $\infty^{(1)}$         \\
    RQ                            &  $\infty$ \\
    $\gamma$-Exp                  & $\frac{\gamma}{2}$  \\
    PP-$q$                            &  $q+\frac{1}{2}$\\
    Dirichlet                     & $\infty$ \\
    \bottomrule
    \end{tabular}
    \footnotesize
    \begin{enumerate}[label=(\arabic*),noitemsep]
        \item We write $\alpha\rightarrow\infty$ to indicate that the RKHS is embedded in $\mathcal{C}^{\infty}$ and thus is contained in any H\"older space with $\alpha<\infty$.
    \end{enumerate}
    \end{table}

\subsubsection{Regret Bounds for Holder-Smooth RKHSs}
 \label{subsubsec:holderregret}
The spectral analysis of these RKHSs opens up the door to a whole family of optimization algorithms that leverage local structure in the function space and come with regret upper bounds that are optimal for some RKHSs. In \cite{liuSmoothBanditOptimization2021}, the UCB-Meta algorithm was proposed, bridging the gap between prior algorithms for $\alpha$-H\"older continuous functions with $\alpha\in(0,1]$ and functions in $\mathcal{C}^{\infty}$. In particular, for functions $f\in\mathcal{C}^{\alpha}$, $\alpha>0$, with samples subject to additive $\sigma$-sub-Gaussian noise, UCB-Meta has an expected cumulative regret upper bounded by $\tilde{\mathcal{O}}\left(n^{\frac{\alpha+d}{2\alpha+d}}\right)$ (Theorem 4, \cite{liuSmoothBanditOptimization2021}), matching the lower bound derived in \cite{wangOptimizationSmoothFunctions2018}. Applying the H\"older smoothness parameters that we derived for specific RKHSs in Table~\ref{table:holderparams}, we obtain the regret bounds for UCB-Meta, including other optimal H\"older smoothness-based algorithms, given in Table~\ref{table:specificregbounds}.

In the next section, we will see that the connection between bandit optimization in H\"older spaces and Besov spaces gives us an opportunity to establish the optimality of these regret upper bounds for the RKHSs that coincide with Besov spaces.

\subsubsection{Besov Space Equivalences}
\label{subsec:analysisbesov}

In this section we show that under additional conditions on the spectral decay rate, we can further characterize the local continuity structure of some RKHSs by strengthening H\"older embedding to norm-equivalence with Besov spaces.

\begin{proposition}[Besov Equivalence for Kernels with Polynomial Spectral Decay]
    \label{prop:besov}
    If $f$ is in the RKHS $\mathcal{H}_k$ associated to a kernel $k$ that has a Fourier transform $\hat{k}(\omega)$ satisfying $C_1(1+\lVert\omega\rVert_2)^{-(\beta+d)}\le \hat{k}(\omega)\le C_2(1+\lVert\omega\rVert_2)^{-(\beta+d)}$ for all $\omega$ and some $C_1,C_2>0$ and $\beta>0$, then there exist constants $C_3,C_4>0$ such that $C_3 \lVert f\rVert_{\mathcal{B}_{2,2}^s}\le \lVert f\rVert_{\mathcal{H}_k}\le C_4\lVert f\rVert_{\mathcal{B}_{2,2}^s}$, where $\mathcal{B}_{2,2}^s$ is a Besov space with smoothness $s=\frac{\beta+d}{2}$ and integrability 2.
\end{proposition}
The proof of this statement is given in Appendix~\ref{appendix:besov}. This result establishes a norm-equivalence between RKHSs with polynomially decaying spectrum and fractional Sobolev spaces, which are norm-equivalent to the Besov space $B_{2,2}^s$ for some smoothness parameter $s$. This equivalence is known for the Mat\'ern RKHS and we show it for the first time for the $\gamma$-exponential RKHS using a combination of the spectral decay results of Proposition~\ref{prop:spectraldecay} and the equivalence result Proposition~\ref{prop:besov}. Note that that the conditions for Proposition~\ref{prop:holder} are satisfied when those of Proposition~\ref{prop:besov} are as well, and so there is a H\"older embedding whenever Besov equivalence holds.

\subsubsection{Regret Bounds for Besov-Smooth RKHSs}
    \label{subsubsec:besovregret}
For Besov spaces, regret lower bounds were established in \cite{singhContinuumArmedBanditsFunction2021}, which revealed that algorithms that are optimal for bandit optimization in H\"older spaces, such as UCB-Meta \citep{liuSmoothBanditOptimization2021}, are also optimal for Besov spaces. In particular, for the Besov space $B_{2,2}^{\frac{\beta+d}{2}}$, the cumulative regret is asymptotically upper and lower bounded by $\tilde{\Theta}\left(n^{\frac{\beta+2d}{2\beta+2d}}\right)$ (Theorem 9, \cite{singhContinuumArmedBanditsFunction2021}). This result gives us the lower bounds for the Mat\'ern, $\gamma$-exponential, and piecewise polynomial RKHSs displayed in Table~\ref{table:specificregbounds}, as well as the matching upper bounds that can be achieved by an algorithm such as UCB-Meta.

\subsection{Global-Local RKHS Optimization via LP-GP-UCB}
\label{subsec:algorithm}
In the previous sections we found that the smoothness of the kernel function may be leveraged in bandit optimization algorithms from two different perspectives, a local approximation one, and a global interpolation one typical in the kernelized bandit literature. In each case, the available samples are used differently, resulting in different surrogate models with different associated performance guarantees. For some kernels, the different approaches result in the same asymptotic regret bounds, while for others, there may be a trade-off between picking one approach over the other. For example, in the cases where characterizing the Mercer eigenvalues of the kernel is more difficult than embedding the RKHS in a H\"older space, it becomes possible to obtain explicit theoretical regret upper bounds using local continuity properties, but in the general absence of regret lower bounds it is not straightforward that the local approach is ever truly better than the global one. 

To unify the two perspectives, we would like to exploit the inherent smoothness of these kernels in an optimization algorithm that simultaneously leverages the local and global properties of the RKHS. Motivated by the generality and situational optimality of H\"older-optimal algorithms, we analyze LP-GP-UCB, an algorithm proposed in \cite{leeMultiScaleZeroOrderOptimization2022a} that augments Gaussian process (GP) surrogate models with LP approximations to exploit the existing smoothness properties of RKHS functions. The model generated by the data in the GP approach is optimal for the assumed RKHS in a global, regularized least-squares sense, which makes it powerful as a surrogate model for global optimization, while the LP approximations exploit the smoothness of the RKHS in local partitions of the search space. 
\subsubsection{Algorithm Overview}
\label{subsec:lpgpucb}
We first summarize the steps of the LP-GP-UCB algorithm \citep{leeMultiScaleZeroOrderOptimization2022a}. 

LP-GP-UCB operates by maintaining an adaptive partition $\mathcal{P}_t$ of the domain $\mathcal{X}$ and constructing upper confidence bounds (UCBs) that leverage both global kernel structure and local smoothness properties. The algorithm takes in as inputs the evaluation budget $n$, kernel $k$, RKHS norm bound $B$, noise parameter $\sigma$, polynomial degree $q$, H\"older exponent $s \in (0,1]$, a H\"older norm bound $L$, and confidence parameter $\delta\in[0,1]$. It defines $\alpha_1 = \max \{s, \min\{1, q\}\}$. 

For each cell $E\in\mathcal{P}_t$, the algorithm constructs a UCB as the minimum of three complementary bounds:
\[U_{t,E}=\min\{u_E^{(0)},u_{t,E}^{(1)},u_{t,E}^{(2)}\},\]
where $u_E^{(0)}$ provides an initial conservative bound that is updated based on the cell sizes and confidence interval widths, and $u_{t,E}^{(1)}$ and $u_{t,E}^{(2)}$ are defined as:
\begin{align*}
u_{t,E}^{(1)} &= \mu_t(x_{t,E}) + \beta_n \sigma_t(x_{t,E}) + L (\sqrt{d}r_E)^{\alpha_1}\\
u_{t,E}^{(2)} &= \hat{\mu}_t(E) + b_t(E) + L (\sqrt{d}r_E)^{\alpha_1}.
\end{align*}

Here $\mu_t$ and $\sigma_t$ are the GP posterior mean and variance, leveraging the global kernel structure, $\hat{\mu}_t(E)$ is an empirical estimate of the cell average, $b_t(E)$ is its confidence interval width, and $L(\sqrt{d}r_E)^{\alpha_1}$ bounds function variation across the cell using H\"older continuity.

At each step, the algorithm selects the cell $E_t$ maximizing the UCB $U_{t,E}$. It then decides to expand the partition when the cell is large and confidence intervals are tight relative to function variation, or sample uniformly at a point $x_t\in E_t$ to gather more information.

When cells become sufficiently small and contain enough observations $\mathcal{D}_{\mathcal{X}}^{(E)}$, the algorithm constructs local polynomial (LP) estimators. For a point $z \in E$, the LP estimator $\hat{f}_E (z, \vec{w}) = \sum_{x  \in \mathcal{D}^{(E)}} w_{x} y_x,$ uses interpolation weights $\vec{w}$ that solve the following problem \citep[Eq.~(1.36)]{nemirovskiTopicsNonParametricStatistics1998}:
\[\vec{w} = \argmin_{\vec{v} = \{v_x \,:\, x \in \mathcal{D}_{\mathcal{X}}^{(E)}\} } \; \sum_{x \in \mathcal{D}^{(E)}_{\mathcal{X}}} |v_x|^2\quad\text{s.t.}\quad p(z) = \sum_{x \in \mathcal{D}_{\mathcal{X}}^{(E)}} v_x p(x)\quad \forall p \in \mathcal{P}_{d}^q,\]
which ensures exact reproduction of $\mathcal{P}_d^q$, the polynomials up to degree $q$, while minimizing estimator variance. If the number of data points in the cell $E$, $|\mathcal{D}_{\mathcal{X}}^{(E)}|$, is larger than $(q+2)^d$, then the problem is solvable and its optimal solution is unique \citep[Lem. 1.3.1]{nemirovskiTopicsNonParametricStatistics1998}. For functions $f$ with $\|f\|_{\mathcal{C}^{q+s}} \leq L$, it is known that $\Phi_k(f,E) \leq L (\sqrt{d}r_E)^{q+s}$ \citep{jonssonFunctionSpacesSubsets1984}. 
This fact and an upper bound on the estimation error between $\hat{f}_E(x,\vec{w})$ and $f(x)$ from \citep[Proposition~1.3.1]{nemirovskiTopicsNonParametricStatistics1998} are used to bound the worst-case error of the LP estimator.

This hybrid UCB structure allows the algorithm to selectively leverage whichever property of the function, local or global, provides tighter error bounds in different regions of the search space. This design enables order-optimal performance across many kernel families without requiring a priori knowledge of which perspective is superior.

\subsubsection{Regret Analysis}
\label{subsec:regretanalysis}
We now develop the main result of this section where we use our information gain and H\"older embedding results to specialize the generic regret bounds for the LP-GP-UCB algorithm, which depend on the \textit{maximum information gain} $\gamma_n$ of the kernel $k$ and the H\"older smoothness parameter $\alpha$. We provide high-probability regret bounds for LP-GP-UCB with specific RKHSs, presented in Table~\ref{table:specificregbounds}. Recall first that $\tilde{\mathcal{O}}(\cdot)$ hides the polylogarithmic factors, and that $\alpha_1 = \max\{s, \, \min\{1, q\}\}$.

\begin{fact}[Regret of LP-GP-UCB \citep{leeMultiScaleZeroOrderOptimization2022a}]
    \label{theorem:regret1} 
Suppose Assumption~\ref{assump:main} holds, and LP-GP-UCB is run with a budget $n$, $q = \lceil\alpha\rceil-1$, $s=\alpha-q$, and inputs as described in Section~\ref{subsec:lpgpucb}. Then with probability at least $1-\delta$ for a given $\delta \in (0,1)$:
\begin{small}
\begin{align}
    \mathcal{R}_n &= \tilde{\mathcal{O}}( \gamma_n \sqrt{n} ). \label{eq:info_type1}
\end{align}
\end{small}
 In addition, the following smoothness-dependent bounds hold for sufficiently large $n$:
 \begin{small}
\begin{align}
   &\mathcal{R}_n = \tilde{\mathcal{O}}( n^{\frac{2\alpha-\alpha_1 + d}{2\alpha+d}}), \; \text{if } \gamma_n =\Omega(\sqrt{n}), \label{eq:smooth1}\\
 &\mathcal{R}_n = \tilde{\mathcal{O}}(n^{\frac{\alpha_1+d}{2\alpha_1 + d}}), \; \text{otherwise.} \label{eq:smooth2}
\end{align}
\end{small}
\end{fact}

Now we state the specific regret bounds for the Mat\'ern, SE, RQ, $\gamma$-Exp, PP, and Dirichlet kernels as a special case of Fact~\ref{theorem:regret1}.

\begin{theorem}[Specific Regret Bounds for LP-GP-UCB Algorithm]
    \label{thm:lpgpucbregret}
    Suppose Assumption~\ref{assump:main} holds, and LP-GP-UCB is run with a budget $n$, $q = \lceil\alpha\rceil-1$, $s=\alpha-q$, and inputs as described in Section~\ref{subsec:lpgpucb}. Then with probability at least $1-\delta$ for a given $\delta \in (0,1)$:
    \begin{itemize}
        \item When $\hat{k}$ decays exponentially or is compactly supported, $\mathcal{R}_n=\tilde{\mathcal{O}}\left(\sqrt{n}\right)$.
        \item When $\hat{k}$ decays at least as fast a polynomial with rate $\tau=\beta+d$ for some $\beta>0$,
        \[\mathcal{R}_n= \begin{cases}
            \tilde{\mathcal{O}}\left(n^{\frac{\beta+2d}{2\beta+2d}}\right), & \beta\le 2\\
            \tilde{\mathcal{O}}\left(n^{\frac{\beta-1+d}{\beta+d}}\right) & \beta>2.
        \end{cases}\]
        The latter bound may be tightened to $\mathcal{R}_n=\tilde{\mathcal{O}}\left(n^{\frac{\beta+3d}{2\beta+2d}}\right)$ when $\hat{k}$ decays like a polynomial with $\beta>2$.
    \end{itemize}
\end{theorem}
\begin{proof}
Proposition~\ref{prop:infogain} gives polylogarithmic upper bounds on $\gamma_n$ for $k$ when $\hat{k}$ is exponential or compactly supported, which can be combined with the information-based regret bound for LP-GP-UCB to give the explicit bound on $\mathcal{R}_n$. 

When $\hat{k}$ decays at least as fast as a polynomial with rate $\beta+d$, $\mathcal{H}_k$ is embedded in a H\"older space of order $\alpha=\frac{\beta}{2}$ due to Proposition~\ref{prop:holder}. Then when $\beta\le2$, both of the smoothness-based regret bounds from Fact~\ref{theorem:regret1} become $\mathcal{O}\left(n^{\frac{\beta+2d}{2\beta+2d}}\right)$, since $\alpha_1=\alpha=\frac{\beta}{2}$ in this case. Otherwise when $\beta>2$, $\alpha_1=1$, and only the looser smoothness-dependent bound can be guaranteed to hold. However, when $\hat{k}$ decays like a polynomial and $\beta>2$, Proposition~\ref{prop:infogain} holds and the information-gain bound $\gamma_n=\tilde{\mathcal{O}}\left(n^{\frac{d}{\beta+d}}\right)$ can be applied, tightening the regret bound to $\tilde{\mathcal{O}}\left(n^{\frac{\beta+3d}{2\beta+2d}}\right)$. 
\end{proof}
The resulting regret bounds for specific kernels are given in Table~\ref{table:specificregbounds}. For LP-GP-UCB, the information-based regret bounds and smoothness-dependent regret bounds come into play in different smoothness parameter regimes. We note that LP-GP-UCB does not always achieve the tightest bounds across all kernels, as it does not optimally use the smoothness information and is linear in $\gamma_n$. However, it does provide a unified algorithm where it can match lower bounds without requiring a priori knowledge of which perspective is superior. While the local polynomial approximations allow the algorithm to achieve order-optimality in the low smoothness regime, the global structure ensures that the performance benefits from higher levels of regularity. 

\section{Discussion}
\label{sec:discussion}

We have shown that the spectral properties of isotropic kernels provide a unifying framework for approaching bandit optimization from both global and local perspectives. By characterizing the Fourier decay rates of the Mat\'ern, square-exponential, rational-quadratic, $\gamma$-exponential, piecewise-polynomial, and Dirichlet kernels, we showed that spectral regularity simultaneously determines performance limits for both kernelized algorithms leveraging global GP surrogates and smoothness-based algorithms that use local polynomial approximations. In particular, we showed that the spectral decay rate determines both the maximum information gain, which governs global interpolation error, and the H\"older smoothness parameters, which govern local approximation error. This duality suggests that the problem of bandit optimization in isotropic RKHSs may reduce to that of bandit optimization in smooth function spaces in general.

Our results also reveal that there are order-optimal algorithms for RKHS optimization that do not rely on the convenient properties of kernel regression in RKHSs. Furthermore, the dependence of the regret for kernelized bandits on kernel regularity suggests that Fourier regularity or smoothness may provide a more powerful and general approach to optimizing RKHS functions. However, the tradeoffs between analytical and computational performance across these different approaches need to be studied further and considered in the algorithm design process. The analysis of LP-GP-UCB demonstrates that hybrid approaches can achieve order-optimality across many kernel families by adaptively leveraging whichever structural property yields tighter error bounds in different regions of the search space. The improvement of the hybrid approach to achieve tighter error bounds for both the global and local approximants is of further interest.

\bibliographystyle{abbrvnat}
\bibliography{refs.bib}

@book{abramowitzHandbookMathematicalFunctions1965,
  title = {Handbook of {{Mathematical Functions}} with {{Formulas}}, {{Graphs}}, and {{Mathematical Tables}}},
  author = {Abramowitz, Milton and Stegun, Irene A.},
  year = {1965},
  publisher = {Dover Publications},
  urldate = {2025-09-02}
}

@article{auerFinitetimeAnalysisMultiarmed2002,
  title = {Finite-Time {{Analysis}} of the {{Multiarmed Bandit Problem}}},
  author = {Auer, Peter and {Cesa-Bianchi}, Nicol{\`o} and Fischer, Paul},
  year = {2002},
  month = may,
  journal = {Machine Learning},
  volume = {47},
  number = {2},
  pages = {235--256},
  issn = {1573-0565},
  doi = {10.1023/A:1013689704352},
  urldate = {2024-12-26},
  langid = {english}
}

@inproceedings{calandrielloGaussianProcessOptimization2019,
  title = {Gaussian {{Process Optimization}} with {{Adaptive Sketching}}: {{Scalable}} and {{No Regret}}},
  shorttitle = {Gaussian {{Process Optimization}} with {{Adaptive Sketching}}},
  booktitle = {Proceedings of the 32nd {{Conference}} on {{Learning Theory}}},
  author = {Calandriello, Daniele and Carratino, Luigi and Lazaric, Alessandro and Valko, Michal and Rosasco, Lorenzo},
  year = {2019},
  volume = {99},
  pages = {533--557},
  address = {Phoenix, USA},
  url = {https://proceedings.mlr.press/v99/calandriello19a.html},
  urldate = {2024-12-30}
}

@inproceedings{camilleriHighDimensionalExperimentalDesign2021a,
  title = {High-{{Dimensional Experimental Design}} and {{Kernel Bandits}}},
  booktitle = {Proceedings of the 38th {{International Conference}} on {{Machine Learning}}},
  author = {Camilleri, Romain and {Katz-Samuels}, Julian and Jamieson, Kevin},
  year = {2021},
  volume = {139},
  pages = {1227--1237},
  url = {https://proceedings.mlr.press/v139/camilleri21a.html},
  urldate = {2024-12-27}
}

@book{jonssonFunctionSpacesSubsets1984,
  title = {Function {{Spaces}} on {{Subsets}} of {{Rn}}},
  author = {Jonsson, Alf and Wallin, Hans},
  year = {1984},
  publisher = {Harwood Academic Publishers},
  googlebooks = {dBKoAAAAIAAJ},
  isbn = {978-3-7186-0128-8},
  langid = {english}
}

@inproceedings{leeMultiScaleZeroOrderOptimization2022a,
  title = {Multi-{{Scale Zero-Order Optimization}} of {{Smooth Functions}} in an {{RKHS}}},
  booktitle = {2022 {{IEEE International Symposium}} on {{Information Theory}} ({{ISIT}})},
  author = {Lee, Madison and Shekhar, Shubhanshu and Javidi, Tara},
  year = {2022},
  month = jun,
  pages = {288--293},
  issn = {2157-8117},
  doi = {10.1109/ISIT50566.2022.9834683},
  urldate = {2025-03-24}
}

@inproceedings{liGaussianProcessBandit2022b,
  title = {Gaussian {{Process Bandit Optimization}} with {{Few Batches}}},
  booktitle = {Proceedings of {{The}} 25th {{International Conference}} on {{Artificial Intelligence}} and {{Statistics}}},
  author = {Li, Zihan and Scarlett, Jonathan},
  year = {2022},
  month = may,
  volume = {151},
  pages = {92--107},
  publisher = {PMLR},
  url = {https://proceedings.mlr.press/v151/li22a.html},
  urldate = {2024-12-27},
  langid = {english}
}

@inproceedings{liuAdaptationMisspecifiedKernel2023b,
  title = {Adaptation to {{Misspecified Kernel Regularity}} in {{Kernelised Bandits}}},
  booktitle = {Proceedings of {{The}} 26th {{International Conference}} on {{Artificial Intelligence}} and {{Statistics}}},
  author = {Liu, Yusha and Singh, Aarti},
  year = {2023},
  volume = {206},
  pages = {4963--4985},
  publisher = {PMLR},
  doi = {10.48550/arXiv.2304.13830},
  url = {https://proceedings.mlr.press/v206/liu23c.html},
  urldate = {2025-01-31}
}

@inproceedings{liuSmoothBanditOptimization2021,
  title = {Smooth {{Bandit Optimization}}: {{Generalization}} to {{Holder Space}}},
  shorttitle = {Smooth {{Bandit Optimization}}},
  booktitle = {Proceedings of {{The}} 24th {{International Conference}} on {{Artificial Intelligence}} and {{Statistics}}},
  author = {Liu, Yusha and Wang, Yining and Singh, Aarti},
  year = {2021},
  volume = {130},
  pages = {2206--2214},
  publisher = {PMLR},
  url = {https://proceedings.mlr.press/v130/liu21f.html},
  urldate = {2025-01-06},
  langid = {english}
}

@book{maternSpatialVariation1986,
  title = {Spatial {{Variation}}},
  author = {Mat{\'e}rn, Bertil},
  year = {1986},
  series = {Lecture {{Notes}} in {{Statistics}}},
  volume = {36},
  publisher = {Springer},
  address = {New York, NY},
  doi = {10.1007/978-1-4615-7892-5},
  urldate = {2025-02-11},
  copyright = {http://www.springer.com/tdm},
  isbn = {978-0-387-96365-5 978-1-4615-7892-5}
}

@article{narcowichSobolevErrorEstimates2006,
  title = {Sobolev {{Error Estimates}} and a {{Bernstein Inequality}} for   {{Scattered Data Interpolation}} via {{Radial Basis Functions}}},
  author = {Narcowich, Francis J. and Ward, Joseph D. and Wendland, Holger},
  year = {2006},
  month = sep,
  journal = {Constructive Approximation},
  volume = {24},
  number = {2},
  pages = {175--186},
  issn = {1432-0940},
  doi = {10.1007/s00365-005-0624-7},
  urldate = {2025-01-31},
  langid = {english}
}

@inproceedings{nemirovskiTopicsNonParametricStatistics1998,
  title = {Topics in {{Non-Parametric Statistics}}},
  booktitle = {{{XXVII Saint-Flour Summer School}} on {{Probability}} and {{Statistics}}},
  author = {Nemirovski, Arkadi},
  year = {1998},
  address = {Saint-Flour, France}
}

@misc{olverNISTDigitalLibrary2025,
  title = {{{NIST Digital Library}} of {{Mathematical Functions}}},
  author = {Olver, F. W. J. and Olde Daalhuis, A. B. and Lozier, D. W. and Schneider, B. I. and Boisvert, R. F. and Clark, C. W. and Miller, B. R. and Saunders, B. V. and Cohl, H. S. and McClain},
  year = {2025},
  month = mar,
  journal = {Digital Library of Mathematics Functions},
  urldate = {2025-09-04},
  howpublished = {https://dlmf.nist.gov/}
}

@article{pruittPotentialKernelHitting1969,
  title = {The {{Potential Kernel}} and {{Hitting Probabilities}} for the {{General Stable Process}} in {{RN}}},
  author = {Pruitt, W. E. and Taylor, S. J.},
  year = {1969},
  journal = {Transactions of the American Mathematical Society},
  volume = {146},
  eprint = {1995174},
  eprinttype = {jstor},
  pages = {299--321},
  publisher = {American Mathematical Society},
  issn = {0002-9947},
  doi = {10.2307/1995174},
  urldate = {2025-02-11}
}

@book{rasmussenGaussianProcessesMachine2006,
  title = {Gaussian Processes for Machine Learning},
  author = {Rasmussen, Carl Edward and Williams, Christopher K. I.},
  year = {2006},
  series = {Adaptive Computation and Machine Learning},
  publisher = {MIT Press},
  address = {Cambridge, Mass},
  isbn = {978-0-262-18253-9},
  langid = {english},
  lccn = {QA274.4 .R37 2006}
}

@inproceedings{salgiaDomainShrinkingBasedBayesian2021,
  title = {A {{Domain-Shrinking}} Based {{Bayesian Optimization Algorithm}} with {{Order-Optimal Regret Performance}}},
  booktitle = {Advances in {{Neural Information Processing Systems}}},
  author = {Salgia, Sudeep and Vakili, Sattar and Zhao, Qing},
  year = {2021},
  volume = {34},
  pages = {28836--28847},
  publisher = {Curran Associates, Inc.},
  url = {https://proceedings.neurips.cc/paper_files/paper/2021/hash/f19fec2f129fbdba76493451275c883a-Abstract.html},
  urldate = {2024-12-27}
}

@unpublished{saloFunctionSpaces2008,
  title = {Function Spaces},
  author = {Salo, Mikko},
  year = {2008},
  url = {https://wiki.helsinki.fi/xwiki/bin/view/mathstatKurssit/57212/Function%20spaces%2C%20fall%202008/},
  urldate = {2025-03-03},
  note = {Lecture notes, University of Helsinki}
}

@article{santinApproximationEigenfunctionsKernelbased2016,
  title = {Approximation of {{Eigenfunctions}} in {{Kernel-based Spaces}}},
  author = {Santin, Gabriele and Schaback, Robert},
  year = {2016},
  month = aug,
  journal = {Advances in Computational Mathematics},
  volume = {42},
  number = {4},
  eprint = {1411.7656},
  primaryclass = {math},
  pages = {973--993},
  issn = {1019-7168, 1572-9044},
  doi = {10.1007/s10444-015-9449-5},
  urldate = {2025-01-06},
  archiveprefix = {arXiv}
}

@book{sawanoTheoryBesovSpaces2018,
  title = {Theory of {{Besov Spaces}}},
  author = {Sawano, Yoshihiro},
  year = {2018},
  series = {Developments in {{Mathematics}}},
  volume = {56},
  publisher = {Springer},
  address = {Singapore},
  doi = {10.1007/978-981-13-0836-9},
  urldate = {2025-10-09},
  copyright = {http://www.springer.com/tdm},
  isbn = {9789811308352 9789811308369}
}

@inproceedings{scarlettLowerBoundsRegret2017,
  title = {Lower {{Bounds}} on {{Regret}} for {{Noisy Gaussian Process Bandit Optimization}}},
  booktitle = {Proceedings of the 2017 {{Conference}} on {{Learning Theory}}},
  author = {Scarlett, Jonathan and Bogunovic, Ilija and Cevher, Volkan},
  year = {2017},
  month = jun,
  pages = {1723--1742},
  publisher = {PMLR},
  url = {https://proceedings.mlr.press/v65/scarlett17a.html},
  urldate = {2024-12-27},
  langid = {english}
}

@inproceedings{schabackApproximationPositiveDefinite2003,
  title = {Approximation by {{Positive Definite Kernels}}},
  booktitle = {Advanced {{Problems}} in {{Constructive Approximation}}},
  author = {Schaback, Robert and Wendland, Holger},
  editor = {Buhmann, Martin D. and Mache, Detlef H.},
  year = {2003},
  pages = {203--222},
  publisher = {Birkh{\"a}user},
  address = {Basel},
  doi = {10.1007/978-3-0348-7600-1\_15},
  isbn = {978-3-0348-7600-1},
  langid = {english}
}

@inproceedings{scholkopfGeneralizedRepresenterTheorem2001,
  title = {A {{Generalized Representer Theorem}}},
  booktitle = {Computational {{Learning Theory}}},
  author = {Sch{\"o}lkopf, Bernhard and Herbrich, Ralf and Smola, Alex J.},
  editor = {Helmbold, David and Williamson, Bob},
  year = {2001},
  pages = {416--426},
  publisher = {Springer},
  address = {Berlin, Heidelberg},
  doi = {10.1007/3-540-44581-1_27},
  isbn = {978-3-540-44581-4},
  langid = {english}
}

@misc{shekharMultiScaleZeroOrderOptimization2020a,
  title = {Multi-{{Scale Zero-Order Optimization}} of {{Smooth Functions}} in an {{RKHS}}},
  author = {Shekhar, Shubhanshu and Javidi, Tara},
  year = {2020},
  month = may,
  number = {arXiv:2005.04832},
  eprint = {2005.04832},
  publisher = {arXiv},
  doi = {10.48550/arXiv.2005.04832},
  urldate = {2025-09-22},
  archiveprefix = {arXiv},
  note = {arXiv preprint arXiv:2005.04832}
}

@inproceedings{singhContinuumArmedBanditsFunction2021,
  title = {Continuum-{{Armed Bandits}}: {{A Function Space Perspective}}},
  shorttitle = {Continuum-{{Armed Bandits}}},
  booktitle = {Proceedings of {{The}} 24th {{International Conference}} on {{Artificial Intelligence}} and {{Statistics}}},
  author = {Singh, Shashank},
  year = {2021},
  month = mar,
  volume = {130},
  pages = {2620--2628},
  publisher = {PMLR},
  url = {https://proceedings.mlr.press/v130/singh21a.html},
  urldate = {2025-03-06},
  langid = {english}
}

@inproceedings{srinivasGaussianProcessOptimization2010,
  title = {Gaussian {{Process Optimization}} in the {{Bandit Setting}}: {{No Regret}} and {{Experimental Design}}},
  booktitle = {Proceedings of the 27th {{International Conference}} on {{Machine Learning}}},
  author = {Srinivas, Niranjan and Krause, Andreas and Kakade, Sham M. and Seeger, Matthias},
  year = {2010},
  pages = {1015--1022},
  publisher = {Omnipress},
  address = {Haifa, Israel},
  url = {https://dl.acm.org/doi/10.5555/3104322.3104451}
}

@article{srinivasInformationTheoreticRegretBounds2012,
  title = {Information-{{Theoretic Regret Bounds}} for {{Gaussian Process Optimization}} in the {{Bandit Setting}}},
  author = {Srinivas, Niranjan and Krause, Andreas and Kakade, Sham M. and Seeger, Matthias W.},
  year = {2012},
  month = may,
  journal = {IEEE Transactions on Information Theory},
  volume = {58},
  number = {5},
  pages = {3250--3265},
  issn = {1557-9654},
  doi = {10.1109/TIT.2011.2182033},
  urldate = {2024-12-26}
}

@inproceedings{vakiliInformationGainRegret2021,
  title = {On {{Information Gain}} and {{Regret Bounds}} in {{Gaussian Process Bandits}}},
  booktitle = {Proceedings of {{The}} 24th {{International Conference}} on {{Artificial Intelligence}} and {{Statistics}}},
  author = {Vakili, Sattar and Khezeli, Kia and Picheny, Victor},
  year = {2021},
  month = mar,
  volume = {130},
  pages = {82--90},
  publisher = {PMLR},
  url = {https://proceedings.mlr.press/v130/vakili21a.html},
  urldate = {2024-12-27},
  langid = {english}
}

@inproceedings{valkoFinitetimeAnalysisKernelised2013a,
  title = {Finite-Time Analysis of Kernelised Contextual Bandits},
  booktitle = {Proceedings of the {{Twenty-Ninth Conference}} on {{Uncertainty}} in {{Artificial Intelligence}}},
  author = {Valko, Michal and Korda, Nathan and Munos, R{\'e}mi and Flaounas, Ilias and Cristianini, Nello},
  year = {2013},
  month = aug,
  series = {{{UAI}}'13},
  pages = {654--663},
  publisher = {AUAI Press},
  address = {Arlington, Virginia, USA},
  url = {https://dl.acm.org/doi/10.5555/3023638.3023705},
  urldate = {2025-11-28}
}

@book{vershyninHighDimensionalProbabilityIntroduction2018,
  title = {High-{{Dimensional Probability}}: {{An Introduction}} with {{Applications}} in {{Data Science}}},
  shorttitle = {High-{{Dimensional Probability}}},
  author = {Vershynin, Roman},
  year = {2018},
  series = {Cambridge {{Series}} in {{Statistical}} and {{Probabilistic Mathematics}}},
  publisher = {Cambridge University Press},
  address = {Cambridge},
  doi = {10.1017/9781108231596},
  urldate = {2025-11-29},
  isbn = {978-1-108-41519-4}
}

@inproceedings{wangOptimizationSmoothFunctions2018,
  title = {Optimization of {{Smooth Functions}} with {{Noisy Observations}}: {{Local Minimax Rates}}},
  shorttitle = {Optimization of {{Smooth Functions}} with {{Noisy Observations}}},
  booktitle = {Advances in {{Neural Information Processing Systems}}},
  author = {Wang, Yining and Balakrishnan, Sivaraman and Singh, Aarti},
  year = {2018},
  volume = {31},
  publisher = {Curran Associates, Inc.},
  url = {https://proceedings.neurips.cc/paper/2018/hash/4ba3c163cd1efd4c14e3a415fa0a3010-Abstract.html},
  urldate = {2025-01-06}
}

@article{watanabeAsymptoticEstimatesMultidimensional2007,
  title = {Asymptotic Estimates of Multi-Dimensional Stable Densities and Their Applications},
  author = {Watanabe, Toshiro},
  year = {2007},
  journal = {Transactions of the American Mathematical Society},
  volume = {359},
  number = {6},
  pages = {2851--2879},
  issn = {0002-9947, 1088-6850},
  doi = {10.1090/S0002-9947-07-04152-9},
  urldate = {2025-01-16},
  langid = {english}
}

@article{wendlandErrorEstimatesInterpolation1998,
  title = {Error {{Estimates}} for {{Interpolation}} by {{Compactly Supported Radial Basis Functions}} of {{Minimal Degree}}},
  author = {Wendland, Holger},
  year = {1998},
  month = may,
  journal = {Journal of Approximation Theory},
  volume = {93},
  number = {2},
  pages = {258--272},
  issn = {0021-9045},
  doi = {10.1006/jath.1997.3137},
  urldate = {2025-11-13}
}

@book{wendlandScatteredDataApproximation2004,
  title = {Scattered {{Data Approximation}}},
  author = {Wendland, Holger},
  year = {2004},
  series = {Cambridge {{Monographs}} on {{Applied}} and {{Computational Mathematics}}},
  publisher = {Cambridge University Press},
  address = {Cambridge},
  doi = {10.1017/CBO9780511617539},
  urldate = {2025-08-29},
  isbn = {978-0-521-84335-5}
}

@book{yaglomCorrelationTheoryStationary2012,
  title = {Correlation {{Theory}} of {{Stationary}} and {{Related Random Functions}}: {{Supplementary Notes}} and {{References}}},
  shorttitle = {Correlation {{Theory}} of {{Stationary}} and {{Related Random Functions}}},
  author = {Yaglom, A. M.},
  year = {2012},
  month = dec,
  publisher = {Springer Science \& Business Media},
  googlebooks = {lgHrBwAAQBAJ},
  isbn = {978-1-4612-4628-2},
  langid = {english}
}

\appendix

\section{Preliminaries}
\label{appendix:definitions}
\begin{itemize}
    \item \textbf{Subgaussianity: } A zero-mean random variable $\eta$ is $\sigma^2$-subgaussian if for all $t\ge 0$, $\mathbb{P}\{|\eta|>t\}\le2e^{-\frac{t^2}{2\sigma^2}}$ \citep{vershyninHighDimensionalProbabilityIntroduction2018}.
    \item \textbf{Reproducing Kernel Hilbert Spaces (RKHS): } Given a positive-definite kernel $k$, we shall use the term $\mathcal{H}_k$ and $\lVert\cdot\rVert_k$ to denote the RKHS associated with $k$ and the corresponding RKHS norm. In particular, $\mathcal{H}_k$ is the completion of the inner product space consisting of functions in the linear span of $k$ and the inner product defined by \[\langle f,g\rangle_k=\sum_{i=1}^{m_1}\sum_{j=1}^{m_2}a_ib_jk(x_i,z_j),\] for functions $f=\sum_{i=1}^{m_1}a_ik(\cdot,x_i)$ and $g=\sum_{j=1}^{m_2}b_j k(\cdot,z_j)$ \citep{rasmussenGaussianProcessesMachine2006}.
    
    \item \textbf{H\"older Spaces: } For $\alpha>0$, we use $\mathcal{C}^{\alpha}$ and $\lVert\cdot\rVert_{\mathcal{C}^{\alpha}}$to denote the H\"older (H\"older-Zygmund) space of order $\alpha$ and the corresponding norm. In particular, $\mathcal{C}^{\alpha}$ contains functions for which the $p$\textsuperscript{th} partial derivatives, $p=\lceil\alpha\rceil-1$, are H\"older continuous with exponent $\alpha-p$ and the derivatives up to and including order $p$ are continuous. \citep{saloFunctionSpaces2008}.

    \item \textbf{Besov Spaces: } For $s>0$ and $1\le p,q\le\infty$, we use $\mathcal{B}^s_{p,q}$ to denote the Besov space with smoothness $s$, integrability parameter $p$, and smoothness scaling parameter $q$ and $\lVert\cdot\rVert_{\mathcal{B}^s_{p,q}}$ to denote the corresponding norm. In particular, $\mathcal{B}^s_{p,q}$ consists of functions in $L^p$ whose  $L^p$ modulus of continuity decays like $t^s$ in $L^q$ norm with respect to $\frac{dt}{t}$. \citep{saloFunctionSpaces2008}.
    \item \textbf{Kernel Functions: } For $k_{\nu}$, the Mat\'ern kernel with parameter $\nu>0$,
    \[k_{\nu}(r)=\frac{2^{1-\nu}}{\Gamma(\nu)}\left(\frac{\sqrt{2\nu}r}{l}\right)^{\nu}K_{\nu}\left(\frac{\sqrt{2\nu}r}{l}\right),\;\nu>0,\]
    where $\Gamma$ is the gamma function and $K_{\nu}$ is the modified Bessel function of the second kind.
    
    For $k_{\text{SE}}$, the square-exponential kernel, 
     \[k_{\text{SE}}(r)=e^{-\frac{r^2}{2l^2}}.\]
    
    For $k_{\text{RQ}}$, the rational-quadratic kernel, 
    \[k_{RQ}(r)=\left(1+\frac{r^2}{2a l^2}\right)^{-a}.\]
    
    For $k_{\gamma-\text{Exp}}$, the $\gamma$-exponential kernel with parameter $\gamma\in(0,2]$, 
    \[k_{\gamma-\text{Exp}}(r)=\exp\left(-\left(\frac{r}{l}\right)^{\gamma}\right),\;0<\gamma\le 2.\]
    
    The piecewise-polynomial functions $k_{\text{PP},q}$ are a family of polynomial kernel functions that have compact support $(-1,1)$ and are $2q$-times continuously differentiable.
    \[k_{\text{PP},q}(r)=\begin{cases}
        \sum_{j=0}^{\lfloor\frac{d}{2}\rfloor+3q+1}c_{j,q}r^j, &0\le r\le 1,\\
        0, & r>1
    \end{cases}.\]
     The minimal-degree polynomial satisfying these constraints and generating a positive definite kernel function has degree $\lfloor\frac{d}{2}\rfloor+3q+1$ and the coefficients $c_{j,q}$ can be computed recursively (Theorem 9.13, \cite{wendlandScatteredDataApproximation2004}).
     
     For $k_{\text{PBL}}$, the Dirichlet kernel,
    \[k_{\text{PBL}}(r)=\frac{\sin((2n+1)x/2)}{(2n+1)\sin(x/2)}=\frac{1}{2n+1}\sum_{k=-n}^ne^{-ikr}.\]
    
\end{itemize}

\section{Spectral Characterizations and Information Gain Bounds}
\label{APPPENDIX:spectra}
\subsection{Proof of Proposition~\ref{prop:spectraldecay}}
\label{appendix:decayproof}
We first show the result for the square-exponential, rational-quadratic, Dirichlet, and Mat\'ern kernels by direct examination of the Fourier transform itself. In addition to showing the appropriate decay, we give the explicit transforms which may be of interest beyond the scope of this work. The kernel definitions used are given in Appendix~\ref{appendix:definitions}.

\textbf{Direct Transform Computation (Square-Exponential, Mat\'ern, Dirichlet)}
The Fourier transform of the square-exponential kernel can be shown to have square-exponential decay.
\[\hat{k}_{\text{SE}}(\omega)=(2\pi l^2)^{\frac{d}{2}}e^{-\frac{l^2\omega^2}{2}}.\]
Thus in the limit, the Fourier transform of the square-exponential kernel is smaller than $C_1\exp(-C_2\lVert\omega\rVert_2)$ for any finite $C_1,C_2>0$. 

Next we consider the Mat\'ern covariance, a generalization of the square-exponential kernel. The Fourier transform is given in \cite{rasmussenGaussianProcessesMachine2006} as:
\[\hat{k}_{\nu}(\omega)=(4\pi)^{\frac{d}{2}}\frac{\Gamma(\nu+\frac{d}{2})}{\Gamma(\nu)}\left(\frac{2\nu}{l^2}\right)^{\nu}\left(\frac{2\nu}{l^2}+\omega^2\right)^{-(\nu+\frac{d}{2})}.\]
Since the coefficients are positive, the Mat\'ern kernel has a Fourier transform with polynomial decay rate $2\nu+d$.

For the Dirichlet kernel, the Fourier transform is bandlimited and thus an extreme case of exponential decay.
\[\hat{k}_{\text{PBL}}(\omega)=\frac{2\pi}{2n+1}\sum_{k=-n}^n\delta(\omega-k),\;n\in\mathbb{N}_0.\]

\textbf{Asymptotic Bounds (Rational-Quadratic, $\gamma$-Exponential), Piecewise-Polynomial}
The exponential eigendecay of the rational-quadratic's Fourier transform comes from its construction as a sum of SE kernels. As observed in \cite{maternSpatialVariation1986}, the RQ kernel is the expectation of the SE kernel with a Gamma$(a,2al)$ distribution on the length-scale parameter. 
\[k_{RQ}(r)=\left(1+\frac{r^2}{2a l^2}\right)^{-a}
    = \int_0^{\infty} e^{-\tau r^2}e^{-\tau(2al)}\frac{(2al)^a}{\Gamma(a)}\tau^{a-1}d\tau,\;a>0.\]
We note that the RQ kernel behaves like the Fourier transform of the Mat\'ern kernel stated above, which is known to decay exponentially as $r\rightarrow\infty$ due to the exponential asymptotic decay of the modified Bessel function of the second kind, $K_{\nu}$ \citep{abramowitzHandbookMathematicalFunctions1965}. To make this connection precise, we compute the Fourier transform directly.
\begin{align*}
    \hat{k}_{\text{RQ}}&=\int_{-\infty}^{\infty}k_{\text{RQ}}(r)e^{-jr\omega}dr\\
    &=\int_{-\infty}^{\infty}\int_0^{\infty} e^{-\tau r^2}e^{-\tau(2al)}\frac{(2al)^a}{\Gamma(a)}\tau^{a-1}e^{-jr\omega}d\tau dr \\
    &=\int_0^{\infty} e^{-\tau(2al)}\frac{(2al)^a}{\Gamma(a)}\tau^{a-1}\int_{-\infty}^{\infty}e^{-\tau r^2}e^{-jr\omega}drd\tau \\
    &= \int_0^{\infty} e^{-\tau(2al)}\frac{(2al)^a}{\Gamma(a)}\tau^{a-1}\sqrt{\frac{\pi}{\tau}}e^{-\frac{\omega^2}{4\tau}}d\tau \\
    &= \frac{(2al)^a}{\Gamma(a)}\sqrt{\pi}\int_0^{\infty}\tau^{a-\frac{3}{2}}e^{-\tau(2al)-\frac{\omega^2}{4\tau}}d\tau \\
    &= \frac{(2al)^a}{\Gamma(a)}\sqrt{\pi}\int_{\infty}^{0}\left(\frac{\omega^2}{4t}\right)^{a-\frac{3}{2}}e^{-\frac{\omega^2(2al)}{4t}-t}\left(-\frac{\omega^2}{4t^2}\right)dt &\text{with $t=\frac{\omega^2}{4\tau}$}\\
    &= \frac{(2al)^a}{\Gamma(a)}\sqrt{\pi}\left(\frac{\omega^2}{4}\right)^{a-\frac{1}{2}}\int_0^{\infty}t^{-a-\frac{1}{2}}e^{-t-\frac{\omega^2(2al)}{4t}}dt \\
    &= \frac{(2al)^a}{\Gamma(a)}2\sqrt{\pi}\left(\frac{\omega}{2}\right)^{a-\frac{1}{2}}K_{a-\frac{1}{2}}(\omega\sqrt{2al}) &\text{\cite{olverNISTDigitalLibrary2025}, 10.32.10.} \\
    &\sim\frac{2\pi(2al)^{a-\frac{1}{4}}}{\Gamma(a)}\left(\frac{\omega}{2}\right)^{a-1}e^{-2\omega\sqrt{2al}}\sum_{k=0}^{\infty}\frac{\left(1-a\right)_k\left(a\right)_k}{k!(-2\omega\sqrt{2al})^k}&\text{\cite{olverNISTDigitalLibrary2025}, 10.40.2.} 
  \end{align*}
  Thus, since $K_{a-\frac{1}{2}}(\omega\sqrt{2al})$, a modified Bessel function of the second kind, decays exponentially in the limit, $\hat{k}_{\text{RQ}}$ also decays exponentially.
  
  For $k_{\gamma-\text{Exp}}$, we are limited to studying the asymptotic decay of the Fourier transform because there is no closed form in terms of elementary mathematical functions aside from the simple case $\gamma=2$, the SE kernel, and $\gamma=1$, the exponential kernel.
 
 As noted in \cite{yaglomCorrelationTheoryStationary2012}, this function is well-studied in probability theory because it is in fact a characteristic function of a L\'evy process with $\gamma$-stable distribution, for which asymptotic density estimates were proposed in \cite{pruittPotentialKernelHitting1969} and proven for the full parameter range $0<\gamma<2$ in \cite{watanabeAsymptoticEstimatesMultidimensional2007}. In particular, as a special case of Theorem 1.5.1 in \cite{watanabeAsymptoticEstimatesMultidimensional2007}, when the spectral measure of a $\gamma$-stable L\'evy process is uniform and continuous, there exist constants $C_1,C_2>0$ such that the density satisfies $C_1(1+\lVert x\rVert_2)^{-(\gamma+d)}\le p(x)\le C_2(1+\lVert x\rVert_2)^{-(\gamma+d)}$ for $x\in\mathbb{R}^d$. Since $k_{\gamma-Exp}(r)$ is the characteristic function of an \textit{isotropic} $\gamma$-stable L\'evy process, the spectral measure is uniform and continuous, and so by duality, its Fourier transform decays polynomially fast with decay rate $\gamma+d$.

For the piecewise polynomial functions, by Theorem 2.1 of \cite{wendlandErrorEstimatesInterpolation1998} there exists a $C_1>0$ such that $\hat{k}_{\text{PP},q}(\omega)\le C_1\lVert\omega\rVert_2^{-2q-1-d}$ for $\lVert\omega\rVert_2>0$. Furthermore, if $q\ge1$ for $d=1,2$, then $\hat{k}_{\text{PP},q}(\omega)\sim (1+\lVert\omega\rVert_2)^{-2q-1-d}$. Thus the piecewise-polynomial kernels have Fourier transforms that decay polynomially with decay rate $2q+1+d$.

\section{Global Interpolation and Information Gain}
\label{appendix:information}

\subsection{Proof of Proposition~\ref{prop:infogain}}
\label{appendix:infogainproof}
 Recall Mercer's theorem (e.g., Theorem 4.2, \cite{rasmussenGaussianProcessesMachine2006}), which states that a positive definite kernel $K$ may be expressed in terms of absolutely summable Mercer eigenvalues $\lambda_i>0$ and eigenfunctions $\phi_i$:
\[k(x,y)=\sum_{i=1}^{\infty}\lambda_i\phi_i(x)\phi_i^*(y).\]
These eigenvalues characterize the fundamental limits of $L_2$ function approximation in finite-dimensional subspaces of RKHSs, and can be bounded using the decay of the kernel's Fourier transform \citep{schabackApproximationPositiveDefinite2003}. 
\begin{fact}[Eigenvalue Upper Bounds (Theorem 6.5,8 \cite{schabackApproximationPositiveDefinite2003})]
\label{fact:eigenbounds}
    The Mercer eigenvalues of a kernel $k$ whose Fourier transform has exponential decay with for a bounded domain satisfy $\lambda_{n+1}\le C_1\exp(-C_2n^{1/d})$ for $n\rightarrow\infty$ and some finite $C_1,C_2>0$.

    The Mercer eigenvalues of a kernel $k$ whose Fourier transform has polynomial decay with rate $\tau=\beta+d$ with $\beta>\frac{d}{2}$ for a bounded domain satisfy $\lambda_{n+1}\le Cn^{-\beta/d}$ for $n\rightarrow\infty$ and some finite $C>0$. This bound may be tightened to $\lambda_{n+1}\le Cn^{-(\beta+d)/d}$ when $\lfloor\frac{\beta+d}{2}\rfloor>\frac{d}{2}$, using improved error estimates from \cite{narcowichSobolevErrorEstimates2006}.
\end{fact}

Thus, our spectral decay results allow us to deduce upper bounds on the kernels' Mercer eigenvalues directly, using this result from \cite{schabackApproximationPositiveDefinite2003} which we strengthen using error bounds from \cite{narcowichSobolevErrorEstimates2006}. Using these eigenvalue tail bounds and the results of Proposition~\ref{prop:spectraldecay}, we can then derive specific information gain upper bounds using the approach of \cite{vakiliInformationGainRegret2021}, where it was shown that one may derive upper bounds on $\gamma_n$ for kernels whose Mercer eigenvalues decay sufficiently rapidly: 
\begin{fact}[Information Gain Upper Bound (Thm. 1, Corr. 1 \citep{vakiliInformationGainRegret2021})]
\label{fact:infobounds}
Let $\delta_D=\sum_{m=D+1}^{\infty}\lambda_m\psi^2$ be the eigenvalue tail mass of a kernel $K$, where $\lambda_m$ are the eigenvalues of the Mercer decomposition, and $\psi$ is an upper bound on the eigenfunction magnitudes. Then the maximum information gain of $k$ satisfies
\[\gamma_n=\mathcal{O}(D\log(n)+\delta_Dn).\]

In particular, if the Mercer eigenvalues of a kernel satisfy $\lambda_{n+1}\le C_1e^{-C_2 n^{1/d}}$, the maximum information gain has an upper bound $\gamma_n=\mathcal{O}(\log^{d+1}(n))$.

    If the Mercer eigenvalues of a kernel satisfy $\lambda_{n+1}\le Cn^{-\beta/d}$ for $n\rightarrow\infty$ and some finite $C>0$, the maximum information gain has an upper bound $\gamma_n=\mathcal{O}(n^{\frac{d}{\beta}}\log^{\frac{\beta-d}{\beta}}(n))$.
\end{fact}

We now prove each case by applying the appropriate eigenvalue decay rate from Fact~\ref{fact:eigenbounds} to the information gain framework of Fact~\ref{fact:infobounds}.

For a kernel whose Fourier transform decays polynomially with rate $\tau=\beta+d$, $\beta>\frac{d}{2}$, on a bounded domain, Fact~\ref{fact:eigenbounds} guarantees that the Mercer eigenvalues satisfy $\lambda_{n+1}\le Cn^{-\beta/d}$ for $n\rightarrow\infty$ and some finite constant $C>0$. Applying the information gain results from Fact~\ref{fact:infobounds}, this gives us
\[\gamma_n=\mathcal{O}\left(n^{\frac{d}{\beta}}\log^{\frac{\beta-d}{\beta}}(n)\right).\]

Under the additional conditions that either $\beta\ge1$ and $d$ is odd, or $\beta\ge 2$, we have $\lfloor\frac{\beta+d}{2}\rfloor>\frac{d}{2}$, which allows us to apply the improved error estimates from \cite{narcowichSobolevErrorEstimates2006} as stated in Fact~\ref{fact:eigenbounds}. This gives us the tightened eigenvalue bound $\lambda_{n+1}\le Cn^{-(\beta+d)/d}$ for $n\rightarrow\infty$ and some finite constant $C>0$. Applying the information gain bounds of Fact~\ref{fact:infobounds} with this improved decay rate, we obtain 
\[\gamma_n=\mathcal{O}\left(n^{\frac{d}{\beta+d}}\log^{\frac{\beta}{\beta+d}}(n)\right),\]
matching the result obtained in \cite{vakiliInformationGainRegret2021} which had been stated for the weaker condition that $\beta>0$.

For a kernel whose Fourier transform has exponential decay, Fact~\ref{fact:eigenbounds} (Theorem 6.8 from \cite{schabackApproximationPositiveDefinite2003}) establishes that the Mercer eigenvalues satisfy
\[\lambda_{n+1} \le C_1 \exp(-C_2 n^{1/d})\]
for $n \to \infty$ and some finite constants $C_1, C_2 > 0$. By the first part of Fact~\ref{fact:infobounds}, kernels with this exponential eigenvalue decay have maximum information gain bounded by
\[\gamma_n = \mathcal{O}(\log^{d+1}(n)).\]

When the Fourier transform of $k$ is compactly supported, the kernel has spectral content limited to a bounded frequency region. This implies super-exponential decay of the eigenvalues, since the effective dimension $D$ of the eigenspace is bounded. For such kernels, the eigenvalue tail mass $\delta_D$ vanishes extremely rapidly, and from Fact~\ref{fact:infobounds}, with $D$ finite or effectively constant, we have
\[\gamma_n = \mathcal{O}(D\log n + \delta_D n) = \mathcal{O}(\log n)\]
since $\delta_D n \to 0$ rapidly and $D$ is bounded.

This completes the proof of all four cases.

\begin{remark}
    In \cite{vakiliInformationGainRegret2021}, the information gain bound for the Mat\'ern~kernel $k_{\nu}$ with $\nu>\frac{1}{2}$ is stated as $\gamma_n=\mathcal{O}(n^{\frac{d}{2\nu+d}}\log^{\frac{2\nu}{2\nu+d}}(n))$. This statement follows from a combination of the authors' information gain bounds for kernels with polynomial eigendecay and earlier works, e.g. \cite{santinApproximationEigenfunctionsKernelbased2016} and \cite{schabackApproximationPositiveDefinite2003}, which assert polynomial eigendecay for translation-invariant kernels with polynomially decaying Fourier transform, an important characteristic of the Mat\'ern~kernel. In particular, it is asserted that when the Fourier transform of the kernel, on a domain with a Lipschitz boundary satisfying an interior cone condition, behaves like $\hat{k}(\omega)\sim(1+\lVert\omega\rVert^2)^{-\frac{2\nu + d}{2}}$ as $\lVert\omega\rVert\rightarrow\infty$, the eigenvalues decay like $\lambda_m=\mathcal{O}(m^{-\frac{2\nu+d}{d}})$ for $m\rightarrow\infty$ (Theorem 6.5, \cite{schabackApproximationPositiveDefinite2003}). This result relies on upper bounds on the distance between functions in $\mathcal{H}_K$ and its interpolants on asymptotically uniformly distributed points, described in Section 4 of \cite{narcowichSobolevErrorEstimates2006} using the isomorphism between $\mathcal{H}_{k_{\nu}}$ and the $L^2$ Sobolev spaces of order $\nu+\frac{d}{2}$. However, these error bounds are shown to hold under certain constraints on $\nu$ and $d$. In particular, $\nu+\frac{d}{2}=l+s$ with $0\le s<1$, $l\in\mathbb{N}$, and $l>\frac{d}{2}$, which implies that we need $\lfloor\nu+\frac{d}{2}\rfloor>\frac{d}{2}$. When the dimension $d$ is odd, this requirement reduces to the known condition that $\nu>\frac{1}{2}$, but when $d$ is even, we require the even stronger condition that $\nu\ge1$. The authors in \cite{narcowichSobolevErrorEstimates2006} note that the error bounds for the undescribed region $\nu\in(0,1)$ were, at the time, an open research problem, and so this particular result is not sufficient for proving the information gain bound for the Mat\'ern~kernel with $\nu\in(\frac{1}{2},1]$.
\end{remark}

\section{Local Smoothness}
\label{appendix:smoothness}    
\subsection{Proof of Proposition~\ref{prop:holder}}
\label{appendix:holder}
Recall the Fourier transform representations of the shift-invariant RKHS $\mathcal{H}_k$ and the order 2 fractional Sobolev space $H^{s,2}$ from the proof of Proposition~\ref{prop:besov}. 

    If $f\in\mathcal{H}_k$, then $f\in L^2(\mathbb{R}^d)$ and is thus a tempered distribution in $\mathbb{R}^n$. Since the Fourier transform decays at least polynomially fast, we have 
    \[\lVert f\rVert_{\mathcal{H}_k}\ge\frac{1}{\sqrt{C_1(2\pi)^{d}}}\lVert(1+\lVert\omega\rVert)^{\frac{\beta+d}{2}}\hat{f}(\omega)\rVert_{L^2}\ge\frac{1}{\sqrt{C_1(2\pi)^{d}}}\lVert f\rVert_{H^{\frac{\beta+d}{2},2}}.\]
    Thus $f$ is contained in the fractional Sobolev space $H^{\frac{\beta+d}{2},2}$.
    By Theorem 3.6.2 in \cite{saloFunctionSpaces2008}, we have the embedding $H^{\frac{\beta+d}{2},2}\subseteq\mathcal{C}^{\frac{\beta}{2}}$, the higher-order H\"older space of smoothness $\frac{\beta}{2}$. 

\subsection{Proof of Proposition~\ref{prop:besov}}
\label{appendix:besov}

Since $k$ is shift-invariant, the corresponding RKHS $\mathcal{H}_k$ has the following Fourier transform representation (Theorem 10.12, \cite{wendlandScatteredDataApproximation2004}):
    \[\mathcal{H}_k=\left\{f\in L^2(\mathbb{R}^d):\lVert f\rVert_{\mathcal{H}_k}=\frac{1}{(2\pi)^{d/2}}\lVert\hat{k}(\omega)^{-1/2}\hat{f}(\omega)\rVert_{L^2} <\infty\right\}.\]
    The fractional Sobolev space $H^{s,p}(\mathbb{R}^n)$ is the set of all tempered distributions $f$ in $\mathbb{R}^n$ such that $\mathcal{F}^{-1}\{(1+\lVert\omega\rVert^2)^{s/2}\hat{f}(\omega)\}$ is in $L^p$, with a norm defined as $\lVert f\rVert_{H^{s,p}}=\lVert \mathcal{F}^{-1}\{(1+\lVert\omega\rVert^2)^{s/2}\hat{f}(\omega)\}\rVert_{L^p}$ \citep{saloFunctionSpaces2008}. By the Plancherel theorem, for the case $p=2$, this is equivalent to the condition that $(1+\lVert\omega\rVert^2)^{s/2}\hat{f}(\omega)$ is in $L^2$, and $\lVert f\rVert_{H^{s,2}}=\lVert (1+\lVert\omega\rVert^2)^{s/2}\hat{f}(\omega)\rVert_{L^2}$. 

    If $f\in\mathcal{H}_k$, then $f\in L^2(\mathbb{R}^d)$ and is thus a tempered distribution in $\mathbb{R}^n$. Since $k$ exhibits polynomial spectral decay, we have 
    \[\lVert f\rVert_{\mathcal{H}_k}\ge\frac{1}{\sqrt{C_2(2\pi)^{d}}}\lVert(1+\lVert\omega\rVert)^{\frac{\beta+d}{2}}\hat{f}(\omega)\rVert_{L^2}\ge\frac{1}{\sqrt{C_2(2\pi)^{d}}}\lVert f\rVert_{H^{\frac{\beta+d}{2},2}}.\]

    Similarly, if $f\in H^{\frac{\beta+d}{2},2}$, $\hat{f}$ and consequently $f$ are $L^2$ integrable, and we have
    \[\lVert f\rVert_{H^{\frac{\beta+d}{2},2}}\ge 2^{-\frac{\beta+d}{2}}\lVert(1+\lVert\omega\rVert)^{\frac{\beta+d}{2}}\hat{f}(\omega)\rVert_{L^2}\ge C_12^{-\frac{\beta+d}{2}}(2\pi)^{d/2}\lVert f\rVert_{\mathcal{H}_K}.\]
    Thus $\mathcal{H}_k$ is norm-equivalent to $H^{\frac{\beta+d}{2},2}$. By Theorem 6 in \cite{saloFunctionSpaces2008}, $H^{\frac{\beta+d}{2},2}$ is norm-equivalent to the Besov space $B_{2,2}^{\frac{\beta+d}{2}}$, and thus $\mathcal{H}_k$ is as well.

\end{document}